\def\year{2018}\relax
\documentclass[letterpaper]{article} 
\usepackage{aaai18}  
\usepackage{amsfonts}		
\usepackage{amsmath}		
\usepackage{amsthm}			
\usepackage{amssymb}		
\usepackage{times}  
\usepackage{helvet}  
\usepackage{courier}  
\usepackage{url}  
\usepackage{graphicx}  
\usepackage{subcaption}	
\usepackage{xfrac}			

\usepackage{tikz}				
\usetikzlibrary{decorations.text}
\usetikzlibrary{arrows,						
								shapes.geometric,	
								shapes.multipart}	

\tikzset{
  stochvar/.style={shape=circle, draw, align=center, inner sep=0.5mm},
  decvar/.style={shape=rectangle, align=center, draw},
  avar/.style={shape=rectangle, rounded corners, align=center, draw},
  vertex/.style={shape=circle, draw, align=center, minimum size=6.25mm, inner sep=1mm},
  arc/.style={>=stealth, shorten >=1pt},
  loarc/.style={>=stealth, shorten >=1pt, ->, dashed},
  hiarc/.style={>=stealth, shorten >=1pt, ->}
}

\graphicspath{ {figures/} }

\usepackage{alltt} 			
\newtheorem{definition}{Definition}
\newtheorem{theorem}{Theorem}

\usepackage{algorithm}
\usepackage[noend]{algpseudocode}	
\algnewcommand\algorithmicinput{\textbf{INPUT:}}
\algnewcommand\INPUT{\item[\algorithmicinput]}

\algnewcommand\algorithmicoutput{\textbf{OUTPUT:}}
\algnewcommand\OUTPUT{\item[\algorithmicoutput]}

\algnewcommand{\IIf}[1]{\State\algorithmicif\ #1\ \algorithmicthen}
\algnewcommand{\EndIIf}{\unskip\ \algorithmicend\ \algorithmicif}
\algnewcommand{\IFor}[1]{\State\algorithmicfor\ #1\ \algorithmicdo}
\algnewcommand{\EndIFor}{\unskip\ \algorithmicend\ \algorithmicfor}
\algnewcommand{\IfThenElse}[3]{
  \State \algorithmicif\ #1\ \algorithmicthen\ #2\ \algorithmicelse\ #3}

\algnewcommand\algorithmicswitch{\textbf{switch}}
\algnewcommand\algorithmiccase{\textbf{case}}

\algdef{SE}[SWITCH]{Switch}{EndSwitch}[1]{\algorithmicswitch\ #1\ \algorithmicdo}{\algorithmicend\ \algorithmicswitch}%
\algdef{SE}[CASE]{Case}{EndCase}[1]{\algorithmiccase\ #1}{\algorithmicend\ \algorithmiccase}%
\algtext*{EndSwitch}%
\algtext*{EndCase}%

\usepackage{cleveref}		
\crefname{corollary}{corrolary}{corrolaries}
\crefname{definition}{definition}{definitions}
\crefname{equation}{equation}{equations}
\crefname{figure}{figure}{figures}
\crefname{lemma}{lemma}{lemmas}
\crefname{section}{section}{sections}
\crefname{table}{table}{tables}

\nocopyright
\begin{document}
%

\title{Stochastic Constraint Optimization using Propagation on Ordered Binary Decision Diagrams}

\author{Anna L.D. Latour \\ LIACS, Leiden University \\ Leiden, The Netherlands \\ \url{a.l.d.latour@liacs.leidenuniv.nl}
	\And	Behrouz Babaki \\ \'Ecole Polytechnique de Montr\'eal \\ Montreal, Canada
	\And	Siegfried Nijssen \\ ICTEAM, Universit\'e catholique de Louvain \\ Louvain-la-Neuve, Belgium \\ \url{siegfried.nijssen@uclouvain.be}}

\maketitle
\begin{abstract}
A number of problems in relational Artificial Intelligence can be viewed as Stochastic Constraint Optimization Problems (SCOPs). These are constraint optimization problems that involve objectives or constraints with a stochastic component. Building on the recently proposed language SC-ProbLog for modeling SCOPs, we propose a new method for solving these problems. Earlier methods used Probabilistic Logic Programming (PLP) techniques to create Ordered Binary Decision Diagrams (OBDDs), which were decomposed into smaller constraints in order to exploit existing constraint programming (CP) solvers. We argue that this approach has as drawback that a decomposed representation of an OBDD does not guarantee domain consistency during search, and hence limits the efficiency of the solver. For the specific case of {\em monotonic} distributions, we suggest an alternative method for using CP in SCOP, based on the development of a new propagator; we show that this propagator is linear in the size of the OBDD, and has the potential to be more efficient than the decomposition method, as it maintains domain consistency.
\end{abstract}

\section{Introduction}
\label{sec:intro}

%

Making decisions under uncertainty is an important problem in business, governance and science. Examples are found in the fields of planning and scheduling, but also occur naturally in fields like data science and bioinformatics. Many of these problems are {\em relational} in nature.

Consider for example a {\em viral marketing problem}~\cite{kempe2003:MaximizingTheSpreadOfInfluence}. We are given a social network of people (vertices) that have stochastic relationships (edges). We  want to rely on word-of-mouth advertisement to turn acquaintances of people who buy our product into new product-buyers. How can we minimize the number of people we need to target directly in a marketing campaign, while a minimum number of people is expected to buy the product?

For another example we are given a network of stochastic protein-gene interactions, with a list of (protein, gene) pairs that are of interest to a biologist~\cite{ourfali2007:spine}. We wish to reduce the network to the part that is relevant for modeling the interactions in that list. This is known as a {\em theory compression problem}~\cite{deraedt2008:CompressingProbLogPrograms}. How can we maximize the sum of interaction probabilities for the interesting pairs, while restricting the number of edges included in the extracted network?

These two problems have common features. First, both problems combine probabilistic networks and decision problems: we either decide who to target in our marketing campaign, or which interactions to select from a protein-gene interaction network. Second, they both involve an objective: minimizing the number of people targeted for marketing and maximizing a sum of probabilities, respectively. Third, they both have to respect a constraint: either reaching a target with respect to the expected number of product-buyers or limiting the number of edges we select for our biologist.

The motivation for our ongoing work is that there is a need for generic tools that can be used to model and solve such problems. In our vision, these tools should combine the state-of-the-art of {\em probabilistic programming} (PP) with {\em constraint programming} (CP).
{\em Probabilistic programming} here provides mechanisms for calculating probabilities of paths in probabilistic networks. 
For making decisions, {\em constraint programming} provides well-established technology.

Note that the stochastic constraint in the viral marketing setting is a hard constraint on a sum of probabilities: we impose a bound on the expected number of people buying the product. This is a different setting than the soft constraints that can be expressed using {\em maximum a posteriori} (MAP) inference or {\em maximum probability estimation} (MPE).

Problems that involve these kinds of hard constraints on probabilities are the focus of the field of {\em stochastic constraint programming} (SCP)~\cite{walsh2002:SCP}, which combines probabilistic inference and constraint programming to solve {\em Stochastic Constraint Optimization Problems} (SCOPs). SCP is closely related to {\em chance constraint programming}~\cite{charnes1959:ChanceCP} and {\em probabilistic constraint programming}~\cite{tarim2009:FindingReliableSolutions}. However, these tools do not provide a modeling language suitable for solving relational problems in a generic manner, and do not link to the probabilistic programming literature.


Recently we proposed a new modeling language and a new tool chain that addresses the problem of modeling and solving relational SCOPs. This language, {\em Stochastic Constraint Probabilistic Prolog} (SC-ProbLog)~\cite{latour2017:CombiningSCOAndPP}, is based on (Decision Theoretic) ProbLog~\cite{deraedt2007:ProbLog,vandenbroeck2010:DTProbLog}, and is therefore particularly suited for modeling probabilistic paths. It extends ProbLog with syntax for specifying SCOPs that are formulated on probabilistic networks, and a tool chain for solving them. Building on ProbLog, SC-ProbLog has functionality for translating a Probabilistic Logic program in Boolean formulas, converting those formulas into Ordered Binary Decision Diagrams (OBDDs) for tractable {\em weighted model counting} (WMC), converting these OBDDs into Arithmetic Circuits (ACs) and decomposing these into Mixed Integer Programs (MIPs), which in turn serve as input for an off-the-shelve MIP solver or CP solver that solves the SCOP.

The main contribution of this paper is a modification of the last step in this pipeline. While in earlier work, constraint optimization solvers were used as a black boxes on decomposed OBDDs, we propose to open the black box in this work.
We will demonstrate that the {\em propagation} that is used in constraint satisfaction solvers, is not optimal for the constraints resulting from decomposition. Specifically, we will show that constraint propagation is not {\em domain consistent}: a search algorithm will branch over variables unnecessarily. To address this flaw, we first introduce a na\"ive propagation algorithm over OBDDs that is domain consistent, and whose worst case complexity is $O(mn)$, where $m$ is the size of the OBDD and $n$ is the number of decision variables. Note that propagation is executed at every node of the search tree; any reduction of $O(mn)$ to a lower complexity affects each node of the search tree. We will then show how to calculate partial derivatives over the OBDDs~\cite{darwiche2003:ADifferentialApproach}, and use these derivatives to reduce the complexity of domain consistent propagation to $O(m+n)$. Here we build on earlier results for linear derivative computation on computational graphs~\cite{iri1984:SimultaneousComputation,rote1990:PathProblemsInGraphs} and computation graphs for the deterministic Decomposable Negation Normal Form (d-DNNF)~\cite{darwiche2001:TractableCounting}. This is a more efficient approach for the calculation of derivatives than the one proposed in~\cite{gutmann2008:LeProbLog}. Furthermore, we will argue that our approach enables the creation of incremental constraint propagation algorithms; this allows for propagation that is more efficient than $O(m+n)$ in practice. Our method assumes the stochastic constraint to have a particular {\em monotonic} property, which we discuss in more detail in \cref{sec:modelling}.

In this paper, we first give a description of how typical SCOPs can be modeled using SC-ProbLog, followed by a discussion on how they can be solved. In \cref{sec:cp} we provide a short introduction to some key concepts of CP, which we use in \cref{sec:approach} to introduce a proposal for an OBDD-based stochastic constraint propagator for CP systems. We conclude this work with an outlook on future research.

\section{Modeling SCOPs with SC-ProbLog}
\label{sec:modelling}

The goal of SC-ProbLog~\cite{latour2017:CombiningSCOAndPP} is to provide a generic system for modeling and solving SCOPs. In this section we give an example SCOP and explain how it can be modeled using SC-ProbLog. Before we address that, let us first define the kinds of SCOP that we consider in this work.

\subsection{Problem Definition}

We consider problems that are defined on two types of variables: {\em decision variables} and mutually independent {\em stochastic variables} (denoted in this work as $d_i$ and $t_i$, respectively). The problems involve a (stochastic) objective function and a set of (stochastic) constraints, all of which can be expressed in terms of these variables. We consider an optimization criterion or constraint to be stochastic if its definition involves stochastic variables. The aim is to find an assignment to the decision variables (also referred to as a {\em strategy}) such that the constraints are respected and the objective satisfied.

In this work we restrict our focus to variables that can take Boolean values. We can assign a value of {\em true} or {\em false} to decision variables, while the value of stochastic variables is mutually independently determined by chance, characterized by an associated probability.

We consider a selection of constraints and objective functions. In particular, we consider constraints that represent a bound on expected utilities and objective functions that maximize or minimize an expected utility, e.g.:
\begin{equation}
	\begin{aligned}
	&\sum_i r_i v_i \geq \theta \quad &\text{stochastic constraint}	\\
	&\max \sum_i r_i v_i	\quad &\text{stochastic optimization criterion}
	\end{aligned}
	\label{eq:cst-opt}
\end{equation}
where $v_i$ either represents the value of a decision variable $d_i$, or a conditional probability $P\left(\phi_i \mid \sigma\right)$. Here $\phi_i$ represents an event, and the conditional probability represents the probability of that event happening (i.e. $\phi_i$ evaluating to {\em true}), given a strategy $\sigma$. With $v_i$ we associate a reward $r_i \in \mathbb{R}^+$, such that the expressions in \cref{eq:cst-opt} represent {\em expected utilities}. For simplicity we will assume $r_i = 1$ in this work, but note that generalizing our approach to $r_i\neq 1$ is trivial. Finally $\theta$ is a {\em threshold} for the constraint. 

Intuitively, in the optimization criterion of the viral marketing problem, $v_i$ represents the probability of the event $\phi_i$ that  person $i$ buys a product, given a marketing strategy. The marketing strategy is represented by decision variables $d_i$.

In this work we impose an additional {\em monotonicity} condition on each probability $P\left(\phi_i \mid \sigma\right)$: we require that for any strategy $\sigma$, switching the value of any decision variable from {\em false} to {\em true}, will yield a probability that is not smaller: $P\left(\phi_i \mid \sigma\right)\geq P\left(\phi_i \mid \sigma'\right)$, if $\sigma'$ differs from $\sigma$ by one variable that is {\em true} in $\sigma'$ but {\em false} in $\sigma$. This condition is met in all the example problems mentioned earlier.

In this work we will consider solving stochastic {\em constraints} rather than stochastic {\em optimization criteria}. However, it is easy to use our results in optimization as well: we can solve a problem involving the {\em optimization criterion} in \cref{eq:cst-opt} by repeatedly solving a constraint satisfaction problem involving the {\em constraint} in \cref{eq:cst-opt}, increasing $\theta$ each time we have found a solution until we find a $\theta$ for which there exists no solution.



\subsection{An Example SCOP}

\begin{figure}
 \centering
 \begin{tikzpicture}
  \node (a) at (-.5,0) {$a$};
  \node (b) at (1,.65) {$b$};
  \node (c) at (1,-.65) {$c$};
  \node (d) at (2.5,0) {$d$};

  \path
		(a) edge node[midway, fill=white] {.7} (b)
				edge node[midway, fill=white] {.4} (c)
				edge node[midway, fill=white] {.8} (d)
		(b) edge node[midway, fill=white] {.5} (d)
		(c)	edge node[midway, fill=white] {.1} (d);

  \end{tikzpicture}
  \caption{A small network of four nodes ($a$, $b$, $c$ and $d$) and five undirected edges with associated probabilities.}
  \label{fig:small-example-network}
\end{figure}
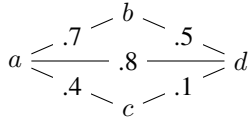

Consider the network in \cref{fig:small-example-network}, and suppose that information can flow through each edge with a certain probability. We can formulate a theory compression problem as described in \cref{sec:intro} on this network. Suppose we want to maximize the sum of probabilities that information can flow from $a$ to $c$ and from $a$ to $d$, but we want to limit the number of edges in the network, such that there are no more than 2 (cardinality constraint). We can model this as follows:
\begin{itemize}
 \item with each edge $(i, j)$ in the network we associate a stochastic variable $t_{ij}$ and a decision variable $d_{ij}$;
 \item with each variable $t_{ij}$ we associate a probability $p\left(t_{ij}\right)$;
 \item the events considered are $\phi_{a\rightarrow c}$ and $\phi_{a\rightarrow d}$, which represent flow of information from $a$ to $c$ and from $a$ to $d$;
 \item our objective is to find a $\sigma$ that maximizes $ P\left(\phi_{a\rightarrow c} \mid \sigma\right) + P\left(\phi_{a\rightarrow d} \mid \sigma\right)$;
 \item our constraint is $\sum_k d_k \leq 2$.
\end{itemize}
Subsequently, we need to define the probability of events $\phi_{a\rightarrow c}$ and $\phi_{a\rightarrow d}$, given a strategy. Here, we use a WMC approach. We use a logical formula  to represent when an event is {\em true}, given an assignment to the decision variables and a sample for the stochastic variables:
\begin{equation}
 \begin{split}
  \phi_{a\rightarrow c} = &\left(d_{ac} \wedge t_{ac}\right) \vee \left(d_{ad} \wedge t_{ad} \wedge d_{cd} \wedge t_{cd}\right) \vee \\
  &\left(d_{ab} \wedge t_{ab} \wedge d_{bd} \wedge t_{bd} \wedge d_{cd} \wedge t_{cd}\right).
 \end{split}
 \label{eq:paths}
\end{equation}
Here, if $t_{ij}$ and $d_{ij}$ are {\em true}, then information can travel through edge $(i,j)$. The logical formula represents all the ways in which information can travel from  $a$ to $c$.

The probability $P\left(\phi_{a\rightarrow c} \mid \sigma\right)$ is then defined as the sum of the probabilities of all the (logical) models of this formula. Given strategy $\sigma=(d_{ac}=d_{ad}=d_{cd}=d_{ab}=d_{bd}=\top)$, one model is for instance $t_{ac}=\top, t_{ab}=t_{ad}=t_{cd}=t_{ab}=\bot$, of which the probability is $.4\cdot (1-.8)\cdot(1-.1)\cdot (1-.5)\cdot (1-.7)$; in principle, we sum the probabilities of all such models to obtain $P\left(\phi_{a\rightarrow c} \mid \sigma\right)$. Note that \cref{eq:paths} has indeed a monotonic property: the more decision variables are {\em true}, the higher the probability of the event is.

To program such a formulas in a generic manner, as well as to define constraints and optimization criteria, we proposed  SC-ProbLog~\cite{latour2017:CombiningSCOAndPP}, which is also based on weighted model counting. The following program in SC-ProbLog would model the problem described above:
{ \scriptsize		
\begin{alltt}
    \emph{% Deterministic facts}
1.  node(a).     node(b).     node(c).     node(d).

    \emph{% Probabilistic facts}
2.  0.7::t(a,b).      0.8::t(a,d).      0.5::t(b,d).
3.  0.4::t(a,c).      0.1::t(c,d).

    \emph{% Decision variables}
4.  ?::d(a,b).        ?::d(a,d).        ?::d(b,d).
5.  ?::d(a,c).        ?::d(c,d).

    \emph{% Relations}
6.  e(X,Y) :- t(X,Y), d(X,Y).  e(Y,X) :- t(X,Y), d(X,Y).
7.  path(X,Y) :- e(X,Y).
8.  path(X,Y) :- X \textbackslash= Y, e(X,Z), path(Z,Y).

    \emph{% Constraints and optimization criteria}
9.  \verb|{ d(X,Y) => 1 :- node(X), node(Y). } 2.|
10. \verb|#maximize { path(a,c) => 1. path(a,d) => 1. }.|
\end{alltt}}
Here, we define the nodes in the network on line 1. Lines 2 and 3 associate the correct probability with each edge; these are the stochastic variables. We define the decision variables in lines 4 and 5. Edges are made undirected in line 6 and we give the definition of a path in lines 7 and 8. 
In line 9 we define the constraint: we assign a utility of 1 to each decision variable that is {\em true}. We also specify that we only allow decision variables that reflect the edges between nodes that are actually present in the network. Finally, line 10 represents the optimization criterion: we assign a utility of 1 to there being a path from $a$ to $c$ and to there being a path from $a$ to $d$. The utilities are summed and weighted by the actual probability of there being such paths. The logical formulas $\phi_{a \rightarrow c}$ and $\phi_{a \rightarrow d}$ are constructed from the program by ProbLog.

An interesting feature of SC-ProbLog is that any problem that does not contain negation or negative weights, represents a monotonic utility function. We restrict our attention in this work to such functions.

In the next section we briefly discuss how to compute the probabilities of such formulas efficiently and how to solve the SCOP of which they are a part.

\section{Solving SCOPs using CP}
\label{sec:solving}

We assume that the reader is familiar with ProbLog\footnote{\url{https://dtai.cs.kuleuven.be/problog}}. In case of absence of that familiarity, we refer the reader to the literature, e.g.~\cite{deraedt2007:ProbLog,fierens2015:InferenceAndLearning}. We start this section with a short recap of why ProbLog uses {\em knowledge compilation} to obtain OBDDs; subsequently, we discuss how OBDDs can be used to solve na\"ively the associated SCOP. Then we discuss the earlier proposed tool chain for solving SCOPs~\cite{latour2017:CombiningSCOAndPP} and reflect on it.

\subsection{From ProbLog to OBDD}
Consider \cref{eq:paths}, and observe that computing $P\left(\phi_{a\rightarrow c} \mid \sigma\right)$ is complicated, as the different paths need to be enumerated, but may also overlap. Therefore, computing this probability involves a disjoint sum problem; in the general case WMC is \#P-complete~\cite{roth1996:HardnessOfApproximateReasoning}. 

In ProbLog the tractability of this task is addressed by compiling the formulas during a preprocessing phase into a Sentential Decision Diagram (SDD)~\cite{darwiche2011:SDD} or OBDD that allows for tractable WMC. The advantage of this method is that, once this diagram is compiled, computing $P\left(\phi \mid \sigma\right)$ has a complexity that is linear in the size of the diagram, thus reducing the complexity of the WMC task (at a cost of having to preprocess the formula). This work focuses on stochastic constraints that can be expressed by OBDDs. We assume familiarity with OBDDs, for we will only discuss a few of their characteristics here. For a more extensive overview, see for example~\cite{ben-ari2012:MathematicalLogicForCS}.

\begin{figure}[ht]
 \centering
 \begin{tikzpicture}
	\newcommand\xscale{1}
	\newcommand\yscale{.9}
 
	\node 					(q) at 		(-1*\xscale, 0*\yscale) 	{$P(\phi_{a\rightarrow c})$};
	\node[stochvar]	(t1) at 	(-1*\xscale, -1*\yscale) {$t_{cd}$};
	\node[decvar] 	(d1) at 	(1*\xscale, -1.5*\yscale) {$d_{cd}$};
	\node[decvar] 	(d2l) at 	(-2*\xscale, -4.5*\yscale) {$d_{ac}$};
	\node[decvar] 	(d2r) at 	(2.5*\xscale, -2*\yscale) {$d_{ac}$};
	\node[stochvar] (t2l) at 	(-1*\xscale, -6.25*\yscale) {$t_{ac}$};
	\node[stochvar] (t2r) at 	(3.75*\xscale, -3*\yscale) {$t_{ac}$};
	\node[stochvar] (t3) at 	(0.75*\xscale, -3.5*\yscale) {$t_{ad}$};
	\node[decvar] 	(d3) at		(3*\xscale, -4.5*\yscale) {$d_{ad}$};
	\node[decvar]		(d4) at		(0*\xscale, -5*\yscale) {$d_{bd}$};
	\node[stochvar]	(t4) at		(1*\xscale, -5.5*\yscale) {$t_{bd}$};
	\node[stochvar] (t5) at		(2.25*\xscale, -6.25*\yscale) {$t_{ab}$};
	\node[decvar] 	(d5) at		(3*\xscale, -7.25*\yscale) {$d_{ab}$};
	
	\node (0) at (-1*\xscale, -8.25*\yscale) {$0$};
	\node (1) at (3.75*\xscale, -8.25*\yscale) {$1$};

  \path 
		(t1) edge[->, color=gray] (q)
		(t1)	edge[hiarc] node[midway,fill=white, inner sep=.5mm] {\footnotesize $.1$}  (d1)
					edge[loarc] node[midway,fill=white, inner sep=.5mm] {\footnotesize $.9$} (d2l)
		(d1)	edge[loarc] (d2l)
					edge[hiarc] (d2r)
		(d2l)	edge[hiarc] (t2l)
					edge[loarc] (0)
		(d2r)	edge[loarc] (t3)
					edge[hiarc] (t2r)
		(t2l)	edge[loarc] node[near start,fill=white, inner sep=.5mm] {\footnotesize $.6$} (0)
					edge[hiarc] node[near end,fill=white, inner sep=.5mm] {\footnotesize $.4$} (1)
		(t2r)	edge[loarc] node[midway,fill=white, inner sep=.5mm] {\footnotesize $.6$}(t3)
					edge[hiarc] node[midway,fill=white, inner sep=.5mm] {\footnotesize $.4$} (1)
		(t3)	edge[loarc] node[midway,fill=white, inner sep=.5mm] {\footnotesize $.2$} (d4)
					edge[hiarc] node[midway,fill=white, inner sep=.5mm] {\footnotesize $.8$} (d3)
		(d3)	edge[loarc] (d4)
					edge[hiarc] (1)
		(d4)	edge[loarc] (0)
					edge[hiarc] (t4)
		(t4)	edge[loarc] node[near start,fill=white, inner sep=.5mm] {\footnotesize $.5$} (0)
					edge[hiarc] node[midway,left, fill=white, inner sep=.5mm] {\footnotesize $.5$} (t5)
		(t5)	edge[loarc] node[near start,fill=white, inner sep=.5mm] {\footnotesize $.3$} (0)
					edge[hiarc] node[midway,left,fill=white, inner sep=.5mm] {\footnotesize $.7$} (d5)
		(d5)	edge[loarc] (0)
					edge[hiarc] (1)
		;
		
	\path (-3*\xscale, -6*\yscale) edge[->] node[midway, above, sloped] {probability} (-3*\xscale, -3*\yscale);

  \end{tikzpicture}
  \caption{An OBDD representing the probability that information can travel from $a$ to $c$ in \cref{fig:small-example-network}, i.e. the probability that $\phi_{a\rightarrow c}$ evaluates to true given any strategy $\sigma$. The variable order corresponding to this OBDD is $t_{cd} < d_{cd} < d_{ac} < t_{ac} < t_{ad} < d_{ad} < d_{bd} < t_{bd} < t_{ab} < d_{ab}$. Circular nodes represent stochastic variables, squares represent decision variables. No specific strategy is reflected here.}
  \label{fig:obdd}
\end{figure}

To see how we can compute $P\left(\phi \mid \sigma\right)$ using an OBDD, consider \cref{fig:obdd}. It shows an OBDD that represents the probability of \cref{eq:paths} evaluating to {\em true}. The weights on the outgoing arcs of nodes that represent stochastic variables (those labeled with $t_{ij}$) correspond to the probability that that variable is {\em true} (for the solid, or {\em hi}, arcs) or {\em false} (dashed, or {\em lo}, arcs). A strategy $\sigma$ is represented in the OBDD by adding weights of 0 and 1 to the outgoing arcs of the nodes corresponding to decision variables (those labeled with $d_{ij}$). For example: if we choose $d_{ac} = \bot$, we put a weight of 0 on the outgoing hi arc of nodes labeled with $d_{ac}$ and weight 1 on their outgoing lo arcs. 

Given a strategy $\sigma$ and arcs labeled accordingly, the OBDD can straightforwardly be mapped to an Arithmethic Circuit (AC). We can compute $P(\phi_{a\rightarrow c} \mid \sigma)$ as follows. In a bottom-up traversal, each OBDD node $r$ takes the value 
\begin{equation}
 v(r) = w(r) \cdot v\left(r^+\right) + \left(1 - w(r)\right) \cdot v\left(r^-\right),
 \label{eq:values}
\end{equation}
where $r^+$ ($r^-$) is the {\em hi} ({\em lo}) child of $r$, i.e. the child connected through the solid (dashed) outgoing arc of $r$; $v(r)=0$ for the negative leaf and $v(r)=1$ for the positive leaf. Observe that $v(root) = P\left(\phi \mid \sigma\right)$.

The complexity of evaluating $P(\phi \mid \sigma)$ is thus linear in the size of the OBDD, but the number of strategies is $2^{n}$, with $n$ the number of decision variables. The na\"ive way of solving a SCOP is to enumerate all possible strategies, use the OBDD to evaluate the objective function and/or constraints for each strategy, evaluate possible other constraints, and store the best feasible strategy found so far. Since the number of strategies is exponential in the number of decision variables, this na\"ive method does not scale well.

\subsection{Solving SCOPs with the SC-ProbLog tool chain}
Since SCOPs are constraint optimization problems, one obvious approach to improving on the na\"ive method is to leverage the state-of-the-art CP solvers that are available. The tool chain described in~\cite{latour2017:CombiningSCOAndPP} takes the OBDD generated by ProbLog and instead of assigning weights to the outgoing arcs of the nodes in the OBDD that represent decision variables, converts the OBDD into an AC in which those weights are present as boolean decision variables. 

A constraint is imposed on the value of the AC, and the then decomposed into a Mixed Integer Program (MIP); a set of smaller constraints is constructed that represent the value at each node of the OBDD according to \cref{eq:values}.
See \cref{fig:problem-obdd} for an example of what such a MIP may look like.

As mentioned in \cref{sec:intro}, this method has a disadvantage: during the search process, the solver cannot guarantee domain consistency on the MIP representing the constraint. We propose an alternative to this decomposition method in \cref{sec:approach}, but will first make the notion of some basic CP concepts, including domain consistency, more concrete.

\section{Introduction to Constraint Programming}
\label{sec:cp}

Constraint programming is an area that studies the development of modeling languages and solvers for constraint satisfaction and optimization problems.
Two processes form the basis of Constraint Programming solvers: {\em search} and {\em propagation}. We briefly discuss these concepts, for they are critical to understanding our contributions in this work. For a more comprehensive overview of CP, we refer the reader to the literature, e.g. {\em Principles of Constraint Programming}~\cite{apt2003:PrinciplesOfConstraintProgramming}. Then we continue with a discussion of the relation between these principles and the circuit decomposition method~\cite{latour2017:CombiningSCOAndPP}.

\subsection{Search and Propagation}
The search process is some structured method for exploring the search space of the problem. In our SCOP setting, the search space consists of all possible assignments to the (binary) decision variables, from which we need to find one that satisfies the constraints and optimizes the objective function.

The details of the search process are outside the scope of this work, but for search over binary variables the process is roughly as follows. Initially, all variables are considered to be free or unassigned; they have a {\em domain} of $\{0,1\}$. Then repeatedly a free decision variable $d$ is selected, and fixed to a value (either {\em true} or {\em false}). After each such assignment, {\em propagation} is used to determine whether other variables can be fixed.  Propagation is the process of updating the domains of other {\em free variables}, making them reflect the consequences of the assignments made to decision variables (the {\em fixed} variables) so far. If propagation yields a contradiction, the search backtracks over the last variable assignment; otherwise, if a free variable remains, its value is fixed and the search process continues.

The constraints of the problem guide the propagation. For example: the problem may contain a {\em cardinality constraint} that puts an upper bound of $N$ on the number of variables that can be set to {\em true}. Suppose that during search, variable $d$ is selected and fixed to {\em true}, becoming the $N$th decision variable to be {\em true}. Now we know that 1 should be removed from the domain of each remaining free variable. This reduces the search space by making domains smaller.

During propagation two things can happen (possibly simultaneously). The first is that the domain of a free variable becomes empty. This means there is no solution given the current partial assignment to decision variables, so we must backtrack to explore a different part of the search space. Alternatively the domain size of a free variable is reduced to 1, leaving only one possible value for that variable (given the current partial assignment). Such a variable can then be fixed and removed from the set of free variables, reducing the search space by reducing the number of free variables.

There are myriad optimizations for both search and propagation, but these are outside the scope of this work. Observe that both the nature of the search and of the propagation depend on the type of variables and on the nature of the constraint. In this work we focus on developing a propagator that enforces {\em domain consistency} on OBDDs.

\subsection{Domain Consistency}

An important notion in propagation is that of {\em domain consistency}. We define it as follows:
\begin{definition}
Let $\varphi(x_1,\ldots,x_n)$ be a constraint over boolean variables $x_1,\ldots, x_n$. Furthermore, let $\sigma'$ be partial assignment to the variables $x_1,\dots,x_n$. Then a propagator for constraint $\varphi$ is {\em domain consistent} if for any $\sigma'$ this propagator calculates a new partial assignment $\sigma$ satisfying these conditions: (1) $\sigma$ extends $\sigma'$, (2) for all variables $v$ not assigned by $\sigma$, both the partial assignments $\sigma \cup \{v\mapsto \bot\}$ and $\sigma \cup \{v \mapsto \top\}$ can be extended to a complete assignment that satisfies the constraint $\varphi$.
\end{definition}
In other words, after domain consistent propagation for a constraint, all values have been removed from all variable domains that cannot be part of a solution for that constraint.

We illustrate this notion with an example. A standard practice in CP is to call the propagator {\em before} the search starts, in order to make the initial domains consistent with the constraint, and, ideally, detect the variables that are forced to a specific value by the constraint.

\begin{figure}[ht]
 \centering
	\begin{tikzpicture}
		\newcommand\xscale{.9}
		\newcommand\yscale{.9}

		\node 															(f) at 		(0*\xscale, 0*\yscale) 	{$P(\phi)$};
		\node[stochvar,  minimum size=5mm]	(r) at 		(0*\xscale, -1*\yscale) {};
		\node[decvar] 											(x) at 		(1.5*\xscale, -2*\yscale) {$x$};
		\node[decvar] 											(y1) at		(-.5*\xscale, -3*\yscale) {$y_1$};
		\node[decvar] 											(y2) at 	(2.5*\xscale, -3*\yscale) {$y_2$};
		\node[stochvar,  minimum size=5mm]	(s) at 		(0.5*\xscale, -4*\yscale) {};
		\node[stochvar,  minimum size=5mm]	(t) at 		(2.5*\xscale, -4*\yscale) {};

		\node (0) at (-.5*\xscale, -5.5*\yscale) {$0$};
		\node (1) at (2.5*\xscale, -5.5*\yscale) {$1$};

		\path
			(r) edge[->, color=gray] (f)
					edge[hiarc] node[midway,fill=white, inner sep=.5mm] {\footnotesize $.9$}  (x)
					edge[loarc] node[midway,fill=white, inner sep=.5mm] {\footnotesize $.1$} (y1)
			(x)	edge[loarc] (y1)
					edge[hiarc] (y2)
			(y1)	edge[hiarc] (s)
						edge[loarc] (0)
			(y2)	edge[loarc] (t)
						edge[hiarc] (s)
			(s)	edge[loarc] node[midway,fill=white, inner sep=.5mm] {\footnotesize $.4$} (0)
					edge[hiarc] node[near end,fill=white, inner sep=.5mm] {\footnotesize $.6$} (1)
			(t)	edge[loarc] node[near start,fill=white, inner sep=.5mm] {\footnotesize $.7$} (0)
					edge[hiarc] node[midway,fill=white, inner sep=.5mm] {\footnotesize $.3$} (1)
			;

		\node at (4.4, -2.5) {$\begin{aligned}
												P(\phi \mid \sigma) &\geq .4	\\
												.1 v(y_1) + .9 v(x) 			&\geq .4			\\
												v(x)		&= (1-x)  v(y_1) + x v(y_2)		\\
												v(y_1) 	&= .6 y		\\
												v(y_2) 	&= .6 y  + .3 (1-y)	\\
												x, y 		&\in \{0,1\}	\\
												&	\\
												0 \leq 	&P(\phi \mid \sigma) \leq .6	\\
												0 \leq 	&v(x) \leq .6	\\
												0 \leq 	&v(y_1) \leq .6	\\
												.3 \leq &v(y_2) \leq .6
		                   \end{aligned}$}
;

	\end{tikzpicture}

  \caption{A small OBDD (left) with three stochastic variables and two decision variables. The two nodes corresponding to decision variable $y$ are indexed for clarity. The MIP on the right is constructed using \cref{eq:values}. 
  }
  \label{fig:problem-obdd}
\end{figure}
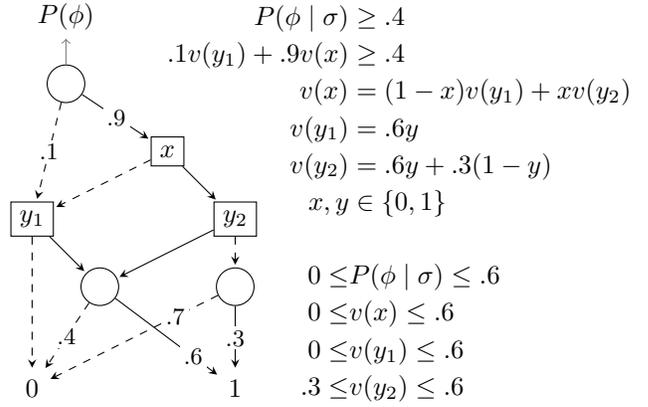

Consider the OBDD in \cref{fig:problem-obdd} and the associated constraint $P(\phi \mid \sigma) \geq .4$. Observe that the four possible strategies yield the following conditional probabilities, which are monotonic in the decision variables:
\begin{equation*}
 \begin{aligned}
  P(\phi \mid x = y = 0) &= 0		& P(\phi \mid x = 1, y = 0) &= .3	\\
  P(\phi \mid x = y = 1) &= .6	& P(\phi \mid x = 0, y = 1) &= .6
 \end{aligned}
\end{equation*}
From this we conclude that only those strategies in which $y = 1$ can possibly satisfy the constraint. A propagator that ensures domain consistency will detect this before the start of the search and fix $y$ to 1.

The circuit decomposition method translates this constraint on the OBDD in a CP model that is also given in \cref{fig:problem-obdd}. Suppose a propagator is called on this decomposed model, before the search starts.  This propagator may start by trying to infer the minimum value $v(y_1)$ needs to take if $v(x)$ takes its maximum possible value. To do this, the propagator assumes for a moment that $v(x) = .6$ holds. Now it can infer that, in order for the constraint to be satisfied, $v(y_1) \geq \sfrac{(.4 - .9 \cdot .6)}{.1} = -1.4$ should hold. Unfortunately, this does not tell us anything, for we already knew that the domain of $y$ is $\{0,1\}$ and thus does not include $-1.4$. Based on this, we cannot remove 0 from the domain of $y$. Repeating this procedure to determine a bound for $v(x)$ by assuming $v(y_1)$ takes it maximum value, and from there continuing to determine bounds for $v(y_1)$ and $v(y_2)$ does not yield conclusive evidence to deduct that $y$ must be fixed to 1, either.

This shortcoming of the circuit decomposition method causes a lack of efficiency, since the search space is not reduced as much as possible. In the next section we introduce a propagator for OBDDs that does ensure domain consistency.



\section{Approach}
\label{sec:approach}
We intend to improve upon the existing circuit decomposition approach for solving SCOPs, by allowing an OBDD-based constraint to be added directly to a CP solver, rather than decomposed into a multitude of (linear) constraints. In order to achieve this, we need to introduce a propagator for OBDDs. As discussed in \cref{sec:cp}, this propagator should guarantee domain consistency in the OBDD.

In this section we will first introduce a na\"ive approach for such a domain consistent propagator. Subsequently, we will show how to obtain a better worst-case complexity by using the idea of {\em derivatives}.

\subsection{Na\"ive Propagator}
As discussed earlier, we can calculate the quality of any strategy with an algorithm that traverses the OBDD bottom-up, using \cref{eq:values}.

For the creation of a domain consistent propagator, our first important observation is that our scoring function is monotonic; hence, the largest possible score is obtained by assigning the value {\em true} to all free decision variables.

The idea behind domain consistent propagation is to repeat the following process for each free decision variable $d$:
\begin{enumerate}
\item fix variable $d$ to the value {\em false};
\item fix all other free variables to the value {\em true};
\item calculate the score for the resulting assignment;
\item if the score is lower than the desired threshold, remove the value {\em false} from the domain of variable $d$.
\end{enumerate}
By construction, this process is domain consistent.

Let $n$ be the number of free decision variables, and let $m$ be the size of the OBDD. Then the complexity of the algorithm above is $O(mn)$: for every free variable we perform a bottom-up traversal of the OBDD. Given that propagation is the most computationally intensive part of search algorithms under our constraint, it is important to obtain a better performance. We will improve this complexity to $O(n+m)$, using an approach similar to that for  d-DNNFs~\cite{darwiche2001:TractableCounting}.

\subsection{Overview of our Propagator}
The key idea behind our improved propagator is that we calculate a derivative
 \begin{equation}
  \frac{\partial f(d, \sigma' \setminus d)}{\partial d} =
  f\left(\sigma'\right) - f\left(d = \bot, \sigma'  \setminus d\right)
 \end{equation}
for every free decision variable $d$. Here,  $\sigma'$ represents a full assignment to all decision variables. In this assignment, every free variable is assumed to have the value {\em true}. Function $f$ represents the function defined by \cref{eq:values} on the root of the OBDD. Hence, $f\left(\sigma'\right)$ represents the best score currently possible, in which all free variables have been given the value $\top$; $f\left(d = \bot, \sigma'  \setminus d\right)$ represents the assignment in which the value for variable $d$ has been switched to $\bot$.

We use the derivative to remove the value {\em false} from the domains of  variables that do not meet this requirement:
 \begin{equation}
	f(\sigma')-\frac{\partial f(d, \sigma' \setminus d)}{\partial d} \geq\theta.
	\label{eq:requirement}
 \end{equation}
Clearly, the main question becomes how to calculate $\sfrac{\partial f(d, \sigma' \setminus d)}{\partial d}$ for all free variables efficiently. Here, we will build on ideas introduced by Darwiche in 2003~\cite{darwiche2003:ADifferentialApproach} to build an $O(m)$ algorithm. This algorithm adapts the ideas of Darwiche to our specific context; we will argue that this enables us to perform propagation for monotonic constraints in an incremental manner, effectively making the complexity lower than $O(m)$.

\subsection{Calculating the Derivative}
\label{subsec:derivative}

We first need to define the concept of {\em path weight}: 
\begin{definition}
 Let $r_m$ be a node labeled with variable $x_m$ in an OBDD with variable order $x_1 < \ldots < x_n$. We define the {\em path weight of $r_m$}:
 \begin{equation}
  \pi(r_m) = \sum_{\ell \in L_{r_m}} \prod_{r_i \in \ell} u_i, \label{eq:pathweight}
 \end{equation}
where $\ell$ is a path from the root of the OBDD to $r_m$, and $L_{r_m}$ is the set of all such paths that are {\em valid}. A path is {\em valid} if it does not include
\begin{itemize}
 \item the hi arc from a node labeled with a decision variable that is {\em false} and
 \item the lo arc from a node labeled with a decision variable that is {\em true} or {\em free}.
\end{itemize}
In other words: we take paths that reflect the current partial assignment, and take the hi arc from {\em free} decision nodes.

In our definition of $u_i$, we use $u_i = 1$ for the outgoing arcs of decision nodes that {\em can} be part of a valid path.

For the outgoing arcs of stochastic nodes labeled with a stochastic variable $x_i$ that has weight $w_i$ as defined in the ProbLog program, we use:
\begin{equation}
u_i = \begin{cases}
				w_i 		& \text{ if we take the hi arc of } r_i; \\
				1 - w_i & \text{ if we take the lo arc of } r_i.
			\end{cases}
\end{equation}

%
Note that the path weight $\pi(r_m)$ is expressed in terms of variables $x_i < x_m$ only.
\end{definition}
An example: if we were to fix $d_{cd}$ and $d_{ac}$ to {\em true} in \cref{fig:obdd}, then the path weight of the node labeled with $t_{ad}$ would be $\pi(r_{t_{ad}}) = .1 \cdot 1 \cdot (1 \cdot .6 + 0) = .06$. 

Our algorithm is based on the observation that derivatives can be calculated using the following equation:
\begin{theorem}
The derivative of the OBDD with respect to a decision variable can be calculated as follows:
\begin{equation}
 \frac{\partial f(d, \sigma'\setminus d)}{\partial d}  = \sum_{r_d \in \text{ OBDD}_d} \pi(r_d)  \left( v(r_d^+) - v(r_d^-)\right),
 \label{eq:partial-derivative-difference}
\end{equation}
where $\text{ OBDD}_d$ represents all nodes in the OBDD labeled with variable $d$.
\end{theorem}

\begin{proof}

Let $f(x_1, \ldots, x_n)$ be the polynomial associated with an OBDD with variable order $x_1 < \ldots < x_n$. Let $r_i$ be a node labeled with $x_i$ and let $w_i$ be the positive weight of that variable. Observe that for any variable $x_m$ (with $x_1 \leq x_m \leq x_n$)
we can write $f$ as
\begin{equation}
	\begin{split}
	 f&(x_1, \ldots, x_m, \ldots, x_n) = \\
	  &\sum_{r_m \in \text{ OBDD}_m} \pi(r_m)  \left( w_m v(r_m^+) + (1-w_m)v(r_m^-)\right),
	\end{split}
\end{equation}
where the values of the the hi and lo child of $r_m$ are $v(x_m^+)$ and $v(x_m^-)$, respectively, following \cref{eq:values}. Recall that in the expression of the path weight of $r_m$, $w_m$ does not occur.
Note also that $v(r_m^+)$ and $v(r_m^-)$ are expressed in variables $x_i > x_m$ only. The derivative of this formula (with respect to $w_m$) corresponds to the claim in the theorem.
\end{proof}

We use the observation above to create an $O(m)$ algorithm for calculating all derivatives in two stages:
\begin{itemize}
\item a top-down pass over the complete OBDD for calculating all path weights;
\item a bottom-up pass for calculating the values for all nodes in the complete OBDD, calculating the derivatives for each variable in the process.
\end{itemize}

\begin{algorithm}[ht]
 \caption{Compute path weights.} \label{algo:path-weights}
 \begin{algorithmic}[1]
	\State $\pi(root) \gets 1$
	\IFor {$r_i \in \text{OBDD}$ and $r_i \neq root$} {$\pi(r_i) \gets 0$} \EndIFor
	\For {each internal node $r_i \in \text{OBDD}$ (in topological order)}
		\If {$r_i$ corresponds to decision variable $d_i$}
			\If {$d_i$ is {\em true} or {\em free}}
				\State $\pi\left(r_i^+\right) \gets \pi\left(r_i^+\right) + \pi(r_i)$
			\Else
				\State $\pi\left(r_i^-\right) \gets \pi\left(r_i^-\right) + \pi(r_i)$
			\EndIf
		\Else
			\State $\pi\left(r_i^+\right) \gets \pi\left(r_i^+\right) + w_i \pi(r_i)$
			\State $\pi\left(r_i^-\right) \gets \pi\left(r_i^-\right) + \left(1 - w_i\right) \pi(r_i)$
		\EndIf
	\EndFor
 \end{algorithmic}
\end{algorithm}

\begin{algorithm}[ht]
 \caption{Compute values.} \label{algo:values}
 \begin{algorithmic}[1]
	\State $v(0) \gets 0$, $v(1) \gets 1$ \Comment Leaf nodes
		\For {each internal node $r_i \in \text{OBDD}$ (in reversed topological order)}
		\If {$r_i$ corresponds to decision variable $d_i$}
			\If {$d_i$ is {\em true} or {\em free}}
				\State $v(r_i) \gets v\left(r_i^+\right)$
			\Else
				\State $v(r_i) \gets v\left(r_i^-\right)$
			\EndIf
		\Else
			\State $v(r_i) \gets w_i v\left(r_i^+\right) + \left(1 - w_i\right) v\left(r_i^-\right)$
		\EndIf
	\EndFor
 \end{algorithmic}
\end{algorithm}

The pseudo codes for these passes are given in \cref{algo:path-weights,algo:values}, respectively.
\begin{algorithm}[ht]
 \caption{Enforce domain consistency.} \label{algo:consistency}
 \begin{algorithmic}[1]
	\For {each {\em free} variable $d$}
		\State $\Delta_d \gets 0$
		\For {each node $r \in \text{OBDD}_d$}
			\State $\Delta_d \gets \Delta_d + \pi(r) \cdot \left(v\left(r^+\right) - v\left(r^-\right)\right)$
		\EndFor
			\If {$v(root) - \Delta_d < \theta$}
				\State remove {\em false} from domain of $d$
			\EndIf
	\EndFor
 \end{algorithmic}
\end{algorithm}

Once these passes are completed, we can recompute the derivatives for all decision variables that are still free, and evaluate \cref{eq:requirement} for each of those to see if we can remove {\em false} from their domain, such that we can enforce domain consistency. The pseudo code for this is provided in \cref{algo:consistency} for clarity, but can be integrated with \cref{algo:values}.
Clearly, the overall calculation finishes in $O(n+m)$ time. 

\subsection{Traversing Part of the OBDD}
\label{subsec:traversing}
For efficient propagation, it is desirable that the complexity of the algorithm above can be reduced more; we should avoid traversing unnecessary parts of the OBDD as much as possible. Building on the ideas presented earlier, some observations allow for more efficient propagation in practice.

As we observed before, the expression for the path weight of an OBDD node labeled with variable $x_m$ (\cref{eq:pathweight}) only contains variables $x_i < x_m$. 
We thus conclude that fixing a decision variable $d$ can only affect the path weights of nodes {\em below} the nodes labeled with that variable $d$.

Moreover: because we take the hi arc both from decision nodes that are {\em free} and from those that are {\em true}, path weights below free decision nodes are not changed at all when we fix a decision node to {\em true}.

Therefore: whenever we fix a decision variable, our propagator only needs to call \cref{algo:path-weights} if we fix it to {\em false}, and even then it only has to traverse the part of the diagram that is below the nodes labeled with that decision variable.

A similar argument holds for the values of the OBDD nodes. Since they are computed in a bottom-up traversal of the OBDD, fixing a variable can only affect the values of the nodes labeled with that variable themselves, and those {\em above} them in the diagram. Again: only fixing a variable to {\em false} actually requires the propagator to update values at all.

We can further narrow down the parts of the diagram that need to be considered. Consider the decision variable that occurs closest to the root of the OBDD. We do not need to maintain the values for any of the nodes in the OBDD above it, as we will never need to calculate the derivative for any variable in this part of the diagram. Similarly, consider the variable closest to the leaves; we do not need to maintain path weights for its descendants either. It can be shown that by only maintaining the part of the OBDD between two borders (the {\em active} part of the OBDD), one can calculate the derivatives exactly, as well as calculate the true value of the optimization criterion without propagating towards the root.

\section{Conclusion and Outlook}

Many problems in AI can be seen as SCOPs. In this work we proposed a new method for solving SCOPs that are modeled using PLP techniques, specifically: SC-Problog~\cite{deraedt2007:ProbLog,latour2017:CombiningSCOAndPP}. In SC-ProbLog, we can convert the SCOP's stochastic constraints into constraints on OBDDs. 
This work was motivated by the observation that an earlier approach was not built on domain consistent propagation. We sketched a propagator for such OBDD constraints that does enforce domain consistency in linear time.

We limited our attention to a representation of distributions in OBDDs. The advantage of this representation is that we can clearly identify parts of the diagram {\em above}  and {\em below } a decision variable; we argued that this can be used to limit the {\em active} part of the diagram and to limit which type of calculation is performed on which part of the diagram.

Several details were omitted from this paper. We did not include extensive details regarding the maintenance of active parts of OBDDs or the incremental calculation of optimization criteria. Furthermore, we did not include an extension to constraints on sums of probabilities, or to constraints of the form $P(\phi \mid \sigma) \leq \theta$.

Concrete next steps are the implementation of our approach, its evaluation on data, a comparison of different approaches for maintaining active parts of OBDDs, and its extension to other types of diagrams. 

\subsubsection{Acknowledgements.} This research was supported by the Netherlands Organisation for Scientific Research (NWO). Behrouz Babaki is supported by a postdoctoral scholarship from IVADO.

\clearpage
\bibliographystyle{aaai}
\bibliography{subfiles/starai2018}

\end{document}